\documentclass[10pt,journal,compsoc]{IEEEtran}
\ifCLASSOPTIONcompsoc
  \usepackage[nocompress]{cite}
\else
  \usepackage{cite}
\fi

\usepackage{times}
\usepackage{helvet}
\usepackage{courier}
\usepackage{amsmath,graphicx,amsthm}
\usepackage{amssymb}
\usepackage{float}
\usepackage{makecell}
\usepackage[ruled,vlined]{algorithm2e}
\usepackage{subfig}

\usepackage{bm}

\newcommand{\bx}{{\mathbf{x}}}

\newcommand{\bbR}{{\mathbb{R}}}

\newcommand{\comment}[1]{}

\def\argmin{\mathop{argmin}\limits}
\def\argmax{\mathop{argmax}\limits}
\def\argmin2{\mathop{argmin}\nolimits}
\def\argmax2{\mathop{\arg\max}\nolimits}

\newtheorem{theorem}{Theorem}
\newtheorem{definition}{Definition}

\usepackage{eso-pic}
\usepackage{xspace}
\makeatletter
\DeclareRobustCommand\onedot{\futurelet\@let@token\@onedot}
\def\@onedot{\ifx\@let@token.\else.\null\fi\xspace}
\def\eg{\emph{e.g}\onedot} 
\def\ie{\emph{i.e}\onedot}

\begin{document}
\title{Guided Co-training for Large-Scale Multi-View Spectral Clustering}

\author{Tyng-Luh Liu \\ Institute of Information Science,
Academia Sinica, Taipei, Taiwan}% <-this % stops a space

\IEEEtitleabstractindextext{%
\begin{abstract}
In many real-world applications, we have access to multiple views of the data, each of which characterizes the data from a distinct aspect. Several previous algorithms have demonstrated that one can achieve better clustering accuracy by integrating information from all views appropriately than using only an individual view. Owing to the effectiveness of spectral clustering, many multi-view clustering methods are based on it. Unfortunately, they have limited applicability to large-scale data due to the high computational complexity of spectral clustering. In this work, we propose a novel multi-view spectral clustering method for large-scale data. Our approach is structured under the guided co-training scheme to fuse distinct views, and uses the sampling technique to accelerate spectral clustering. More specifically, we first select $p$ ($\ll n$) landmark points and then approximate the eigen-decomposition accordingly. The augmented view, which is essential to guided co-training process, can then be quickly determined by our method. The proposed algorithm scales linearly with the number of given data. Extensive experiments have been performed and the results support the advantage of our method for handling the large-scale multi-view situation.
\end{abstract}

}

\maketitle

\IEEEdisplaynontitleabstractindextext

% \fi
\IEEEpeerreviewmaketitle

%%%%%%%%%%%%%%%%%%%%%%%%%%%%%%%%%%%%%%%%%%%%%%%%%%%%%%%%%%%%%%%%%%%%%%%%%%%%%%%%%%%%%%%%%%%%%%%%%%
%
%
\IEEEraisesectionheading{\section{Introduction}\label{sec:introduction}}

\IEEEPARstart{M}ulti-view clustering has been an important research topic in machine learning, data mining, pattern recognition and many other fields. The technique is especially useful for dealing with practical real-world applications with data collected in various forms or described by different feature representations. For example, in computer vision, it is common to use multiple feature descriptors, such as color and texture histograms, to simultaneously characterize an image so that essential information of different aspects can be better captured. The other critical development pertaining to designing a clustering method is the need to more effectively handle extremely large amount of data. To  address these two challenging issues, we propose a co-training strategy for multi-view spectral clustering that is applicable to the large-scale scenario and produces state-of-the-art results.

The main advantage of working with spectral clustering, \eg, \cite{shi2000normalized}, is the flexibility in accommodating a wide range of geometries, including non-convex patterns, in the data distribution \cite{ng2002spectral,von2007tutorial,filippone2008survey}, which would generally lead to better clustering results than using, say, $k$-means clustering. As a result, quite a number of multi-view extensions are developed based on spectral clustering. A notable group of these approaches can be identified by the core idea to fuse the affinity matrices generated from different views to form a ``universal'' affinity matrix for spectral clustering \cite{cai2011heterogeneous,huang2012affinity,guo2014multiple}. Another direction of multi-view spectral clustering, including \cite{zhang2015low} and \cite{cao2015diversity}, accomplishes the task by uncovering the underlying latent space for all given views using the technique of subspace learning. In addition, the task of multi-view clustering can also be addressed by maximizing the mutual agreement of the provided multiple views. A popular formulation, which is more relevant to our method, is to cast the problem as a co-training/co-regularization framework \cite{kumar2011co,kumar2011cor,wang2014multi,lee2016guided}.

However, spectral clustering does not scale well due to its high computational complexity. For $n$ given data points, spectral clustering first constructs an $n\times n$ affinity matrix and then performs eigen-decomposition on the corresponding Laplacian matrix \cite{chung1997spectral}. If the feature dimension is $d$, then the first step takes $\mathcal{O}(n^2d)$ time, while the cost of eigen-decomposition is $\mathcal{O}(n^3)$. Both of them are unbearable for large-scale data. To alleviate the demanding computation complexity, much effort has been made to accelerate/approximate spectral clustering. One feasible direction is to speed up the eigen-decomposition step. To that end, \cite{lin2010power,boutsidis2013spectral} use the power method and \cite{fowlkes2004spectral} consider the classical Nystr\"{o}m method.
Another promising way is to {\em pre-process} the data by sampling techniques and then perform corresponding operations on the subset of a reduced size. Up to our knowledge, \cite{shinnou2008spectral} is the first algorithm adopting this strategy. Subsequently, a more general framework based on the similar idea is proposed in \cite{yan2009fast}. Other approaches in this category include \cite{sakai2009fast,wang2009approximate,chen2011large,shao2015deep,li2016scalable} and \cite{tremblay2016compressive}.

While the aforementioned attempts to accelerate spectral clustering are mainly designed to handle single-view data, the study in \cite{li2015large} seems to be, so far, the only one focusing on multi-view spectral clustering for large-scale data. In comparison, as we shall describe later, our method is motivated by \cite{lee2016guided}, and aims to further formalize and extend the guided co-training framework to facilitate large-scale multi-view spectral clustering. Similar to the relevant techniques, we sample, say, $p$ landmark points from the dataset by considering information from all views, and then explore the nearest-neighbor and similarity relationships between the original data and the landmark points. The strategy would give rise to a sparse representation matrix, which is essentially a rectangular sub-matrix of the complete affinity matrix of an extremely large size. Based on \cite{lee2016guided}, we propose a new criterion to construct an augmented view and update the affinity matrix of each view accordingly. Once all the original views are adjusted, the whole process is repeated for a certain number of iterations or until converging to a consensus view. Finally, spectral clustering can be readily performed with the resulting affinity matrix associated with the augmented view.

%%%%%%%%%%%%%%%%%%%%%%%%%%%%%%%%%%%%%%%%%%%%%%%%%%%%%%%%%%%%%%%%%%%%%%%%%%%%%%%%%%%%%%%%%%%%%%%%%%
%
\section{Related Work}
\label{sec:related}
%
%%%%%%%%%%%%%%%%%%%%%%%%%%%%%%%%%%%%%%%%%%%%%%%%%%%%%%%%%%%%%%%%%%%%%%%%%%%%%%%%%%%%%%%%%%%%%%%%%%

%
\subsection{Large-Scale Spectral Clustering}
\label{ssec:large-scale}
Most of the large-scale spectral clustering approaches fall into two categories. The first category includes those that are designed to speed up the eigen-decomposition. Since this is  the most time-consuming step of spectral clustering, it makes sense to accelerate the process in dealing with large-scale data. \cite{fowlkes2004spectral} utilize the classical Nystr\"{o}m method to approximate the eigenvectors. In \cite{lin2010power}, an efficient algorithm, which uses the power method to approximate the eigenvectors of affinity matrix by an iterative procedure, is proposed. As the Spielman and Teng solver (ST-solver) \cite{spielman2004nearly} is a near-linear time method to solve generalized eigen-problems for a specific class of matrices, it is applied in \cite{khoa2012large} to achieve faster spectral clustering. Deep neural networks can also be used to replace eigen-decomposition. \cite{tian2014learning} discover that the reconstruction objective of auto-encoder is closely related to the graph-construction nature of eigen-decomposition. They specifically implement a sparse auto-encoder to speed up spectral clustering.

Methods in the other category rely on the sampling strategy. The main idea is to reduce the data size by sampling and then approximate the result accordingly. To approximately perform large-scale spectral clustering, \cite{yan2009fast} utilize $k$-means clustering to find the centers of $p$ groups and then carry out spectral clustering only on the $p$ centers. A slightly different strategy is presented in \cite{shinnou2008spectral}. They propose to reduce the data size by applying $k$-means clustering to find $p$ centers and removing the data points close to these centers. Spectral clustering is then performed on the remaining data points for the final result. \cite{chen2011large} use $k$-means clustering or random selection to choose $p$ landmark points and construct a sparse representation matrix by computing the similarity between raw data and landmark points. The eigen-decomposition is approximated by singular value decomposition over the resulting sparse representation. More recently, a similar approach can be found in \cite{li2016scalable}, where they first select $p$ so-called anchor points and then construct a bipartite graph with two types of nodes corresponding to the raw data and anchor points, respectively. Singular value decomposition is also applied to approximate eigen-decomposition.

\subsection{Multi-View Spectral Clustering}
\label{ssec:multi-view}
An intuitive way for achieving multi-view spectral clustering is to find an optimal linear or non-linear combination of different views. \cite{huang2012affinity} combine the affinity matrices obtained from different views by optimizing the weights. The proposed method in \cite{guo2014multiple} aggregates the normalized Laplacian matrices from each view through multiple kernel learning. \cite{xia2014robust} employ a probabilistic model to fuse the normalized Laplacian matrices from multiple views.

Other research efforts leverage with the subspace learning framework to tackle multi-view spectral clustering. Specifically, multi-view subspace learning assumes the existence of a shared latent representation for reconstructing all views. The goal of learning is to recover the latent space so that the data can be described by a lower-dimensional representation. Relevant algorithms in this class include, \eg, \cite{liu2013multi,zhang2015low} and \cite{cao2015diversity}. Both methods derive a low-rank representation for data points via subspace learning and then construct an affinity matrix for multi-view spectral clustering.

It has been noticed that techniques established based on the co-training/co-regularization paradigm for multi-view spectral clustering have gained great success.  The main idea of co-training is to construct separate learners and minimize the disagreement between them, while the co-regularization can be seen as a regularized version of co-training \cite{sindhwani2005co}. In \cite{kumar2011cor}, they employ two co-regularization strategies, pairwise and centroid regularization, to develop two multi-view spectral clustering schemes. Other approaches related to co-regularization include \cite{cai2011heterogeneous} and \cite{wang2014multi}. More relevant to our approach, \cite{kumar2011co} use co-training to make the affinity matrices to agree with each other, and then perform spectral clustering on the consensus view, which is believed to be the most informative view a priori. However, in practice, we do not always have such prior knowledge. Moreover, the co-training procedure may converge to a unified but compromised one. To overcome such disadvantages, \cite{lee2016guided} propose an iterative process of guided co-training, which first constructs an augmented view and then use it to separately guide the improvement of each view.

%%%%%%%%%%%%%%%%%%%%%%%%%%%%%%%%%%%%%%%%%%%%%%%%%%%%%%%%%%%%%%%%%%%%%%%%%%%%%%%%%%%%%%%%%%%%%%%%%%
%
\section{Our Method}
\label{sec:method}
%
%%%%%%%%%%%%%%%%%%%%%%%%%%%%%%%%%%%%%%%%%%%%%%%%%%%%%%%%%%%%%%%%%%%%%%%%%%%%%%%%%%%%%%%%%%%%%%%%%%

%
\subsection{Sparse Affinity Matrix Construction}
\label{ssec:sparse}
As we have mentioned, one direction of reducing the computational cost of large-scale spectral clustering is using a sampling technique to approximate the affinity matrices and the corresponding eigen-decomposition. Our approach follows this line to extract the representative information from the affinity matrix for each view. For $n$ given data points, the idea is to sample $p$ (representative) examples, which are usually named as landmark points, to capture the underlying manifold structure. Among the various sampling techniques, we choose to use $k$-medoids clustering to generate the $p$ landmark points because studies suggest that lightweight clustering methods often produce {\em better} landmark points \cite{wang2009approximate,chen2011large,kumar2012sampling}.

However, with multi-view data, we could not simply perform $k$-medoids clustering on each individual view in that it would generate different landmark points for different views and the co-training process becomes unfeasible. Analogous to the scheme used in \cite{li2015large}, the data clustering is carried out over the concatenated features from all views. More precisely, we denote the data matrix under view $v \in \{1, \dots, V\}$ as
$X^v = [\bx_1^v, \dots, \bx_n^v ] \in \bbR^{d_v \times n}$ and use $X = X^1 \oplus \cdots \oplus X^V \in \bbR^{d \times n}$ to express the data matrix combining all views where $d = \sum_{v=1}^V d_v$ and $\oplus$ symbolizes feature concatenation by column-wise stacking. We carry out $k$-medoids clustering over $X$ to find $p$ centers to form the set of landmark points, denoted as $M=\{\bx_{m_j}\}_{j=1}^p$ where $1\leq m_j \leq n$ is the index of the $j$th landmark point in the raw dataset.

After the $p$ landmark points are selected, we can derive the sparse representation matrix $Z^v \in \bbR^{n\times p}$ for $v = 1,\ldots,V$ by constructing a $q$-NN graph ($q < p$)  between the raw data points and landmark points with the edge weight defined by
\begin{equation}\label{knn}
z^{v}_{ij} = \begin{cases}
	\exp\left(\frac{-\|\bx^v_i - \bx^v_{m_j}\|^2}{2 \sigma^2}\right), & m_j\in \Phi_i,\\
    0, & \text{Otherwise,}
    \end{cases}
\end{equation}
\noindent where $\sigma$ is the band-width parameter and  $\Phi_i$ denotes the set of (raw data) indices of the $q$ nearest neighbors of $\bx_i$ in the landmark set $M$. Following \cite{chen2011large}, we further divide each $z^v_{ij}$ by the corresponding row sum. That is,  each row of the resulting $Z^v$ will now sum to one. Let $D^v \in \bbR^{p\times p}$ be a diagonal matrix with diagonal elements comprising the column sums of $Z^v$, we can then finalize the {\em normalized} sparse representation matrix $\hat{Z}^v$ by
\begin{equation}
\hat{Z}^v = Z^v \left(D^v \right)^{-1/2}.
\label{normalization}
\end{equation}

\subsection{Augmented View}
\label{ssec:augmented}
In \cite{lee2016guided}, a guided co-training approach is established for multi-view spectral clustering. The crux of their method is the use of an {\em augmented view}. While the construction the additional view has been justified in a heuristic way, it is not clear how to generalize the technique to handle large-scale data. We instead give a slightly different definition of the augmented view, and also a theoretical justification to support its usefulness for the large-scale situation.

With (\ref{normalization}), we define the (approximate) graph Laplacian of the $v$th view, $1\le v \le V$, by
\begin{equation}
L^v = \hat{Z}^v \left(\hat{Z}^v \right)^T.
\label{eqn:laplacian}
\end{equation}
Suppose that the multi-view spectral clustering is to divide the data into $k \ll p$ clusters. We first express
the Singular Value Decomposition (SVD) of $\hat{Z}^v$ by
\begin{equation}
\hat{Z}^v = \hat{U}^v \Sigma^v (\hat{V}^v )^T
\label{eqn:svd}
\end{equation}
where $\Sigma^v = \mathrm{diag}(\sigma^v_1,\ldots,\sigma^v_p) \in \bbR^{p \times p}$ with $\sigma^v_1\geq \sigma^v_2\geq\ldots\geq\sigma^v_p$, $\hat{U}^v = [u_1,\ldots,u_p]\in\bbR^{n\times p}$ and $\hat{V}^v = [v_1,\ldots,v_p] \in \bbR^{p\times p}$. Then, the {\em largest} $k$ eigenvectors of graph Laplacian $L^v$ are given by the first $k$ columns of $\hat{U}^v$, which form the reduced representation matrix, denoted as $U^v = [u_1,\ldots,u_k] \in \bbR^{n \times k}$, of the $v$th view. Our goal is to learn an augmented view whose reduce representation encompasses as much information from  the $U^1,\ldots,U^V$ as possible. To further elaborate how the augmented view is constructed, we need the following definition.

\begin{definition}
With (\ref{eqn:laplacian}) and (\ref{eqn:svd}), we denote the approximate affinity matrix for the $v$th view as $A^v$, $1 \le v \le V$, which can be defined based on the corresponding reduced representation $U^v$ by $A^v = U^v (U^v)^T \in \bbR^{n\times n}$.
\end{definition}

Let $A^* = U^* (U^*)^T$ be the augmented view we seek to establish, where $U^* \in \bbR^{n \times k}$ is the corresponding reduced representation matrix. It is preferable to construct $A^*$ such that the augmented view best ``balances'' the disagreement between all views. That is, we consider constructing $A^* = U^* (U^*)^T$ by minimizing the following objective:
\begin{equation}
A^* =U^* (U^*)^T= \arg\min_{A=UU^T} \sum\nolimits_{v=1}^V\left\|A - A^v\right\|_F^2
\label{eqn:A}
\end{equation}
\noindent where we use the Frobenius norm to determine the (squared) distance between the augmented view and the $v$th view.

We now show that solving (\ref{eqn:A}) is equivalent to finding a reduced representation $U^*$ that accounts for as much information from those of the $V$ views. Since $A^* = U^* (U^*)^T$, the augmented view can be determined by
\begin{equation}
U^* = \arg\min_{U^TU = I}\sum\nolimits_{v=1}^V\left\|UU^T - U^v (U^v)^T\right\|_F^2.
\label{centroid}
\end{equation}

\begin{algorithm}[th]
\label{algo1}
  \SetAlgoLined
  \SetKwData{Left}{left}\SetKwData{This}{this}\SetKwData{Up}{up}
  \SetKwFunction{Union}{Union}\SetKwFunction{FindCompress}{FindCompress}
  \SetKwInOut{Input}{input}\SetKwInOut{Output}{output}

  \Input{Multi-view data $X^v \in\bbR^{d_v \times n}$ for $1\leq v\leq V$;\\
         Number of clusters $k$;\\
         Number of Landmarks $p$;\\
         Number of nearest neighbors $q$.}
  \Output{$k$ clusters}
  %initialization\;
  \Begin{
  1. Get $p$ landmark points by $k$-medoids sampling.\\
  2. Construct $\{Z^v \in \bbR^{n\times p}\}_{v=1}^V$ by (\ref{knn}). \\
  \Repeat{Converge}{
    3. Compute the normalized $\{ \hat{Z}^v\}_{v=1}^V$ by (\ref{normalization}).\\

    4. Compute $U^*$ and $A^*$ by solving (\ref{centroid}).\\

    5. Construct ${\tilde A}^*$ by (\ref{sub_aff}).\\
    	\For{$v=1$ to $V$}{
	    \[
    	Z^v \leftarrow {\tilde A}^* \odot Z^v
	    \]
    	}
    }
    6. Run $k$-means clustering on the final $U^*$.
  }
\caption{Guided Co-training for Large-Scale Multi-View Spectral Clustering}
\end{algorithm}

\begin{theorem}
Let $\tilde{U} = \left[ U^1 \ \  U^2 \ \ \cdots \ \ U^V\right]\in\mathbb{R}^{n\times Vk}$, then the solution of (\ref{centroid}) is exactly the $k$ largest left-singular vectors of $\tilde{U}$.
\end{theorem}

\begin{proof}
From the fact $\|B\|_F^2 = tr(BB^T)$ and the cyclicity property of matrix trace, we have
\begin{align}
  & \arg\min_{U}\sum_{v=1}^V\left\|UU^T - U^v (U^v)^T\right\|_F^2 \notag\\
= & \arg\min_{U}\sum_{v=1}^V\left(\|U^v\|_F^2 + \|U\|_F^2 - 2tr\left(U^v(U^v)^TUU^T\right)\right) \notag\\
= & \arg\max_{U}\sum_{v=1}^V\, tr\left(U^v(U^v)^TUU^T\right) \notag\\
= & \arg\max_{U}\, tr\left(U^T\left(\sum_{v=1}^VU^v(U^v)^T\right)U\right)
\end{align}
From the construction, we have $\sum_{v=1}^VU^v(U^v)^T = \tilde{U} \tilde{U}^T$. By the Rayleigh-Ritz theorem, the solution can be obtained by computing the $k$ largest eigenvectors of $\tilde{U} \tilde{U}^T$, which are exactly the $k$ largest left-singular vectors of $\tilde{U}$. In other words, the reduced representation $U^*$ of the optimal augmented view comprises the $k$ largest left-singular vectors of $\tilde{U}$, and is also the best rank-$k$ approximation to $\tilde{U}$.
\end{proof}

\subsection{Guided Co-training}
\label{ssec:guided}
With $A^* \in \bbR^{n\times n}$ of the augmented view so defined, we are ready to proceed the large-scale guided co-training to update the sparse representation matrix $Z^v \in \bbR^{n\times p}$ defined in (\ref{knn}).  Since the $(i,j)$ entry of $Z^v$ is the similarity between data point $\bx_i$ and the $j$th landmark point, we extract a submatrix ${\tilde A}^* \in\bbR^{n\times p}$ from $A^*$ by letting
\begin{equation}\label{sub_aff}
{\tilde A}^*_{ij} = A^*_{i m_j}, \quad j = 1,\ldots,p,
\end{equation}
\noindent where $m_j$ is the index of the $j$th landmark point in the original dataset. Then, we replace $Z^v$ of view $v=1,\ldots,V$ with ${\tilde A}^*\odot Z^v$, where $\odot$ denotes the {\it Hadamard product}.

Having updated all the sparse representation matrices $Z^v, v=1, \dots, V$, we again use the same method to compute $A^*$ and iterate the process until the change of $A^*$ is insignificant (or reaching a pre-specified number of iterations). We summarize the whole algorithm in Algorithm 1.

\subsection{Complexity Analysis}
\label{ssec:complexity}
We compare the proposed algorithm with the original guided co-training \cite{lee2016guided}. Assume $n$ data points of $V$ views, denoted as $X^v \in \bbR^{d_v \times n}, v=1,\dots, V$, are given. First, using $k$-medoids clustering to select $p$ landmark points takes $\mathcal{O}(t_1npd)$ where $t_1$ is the number of iterations of $k$-medoids clustering and $d = \sum_{v=1}^V d_v$. The second step is to construct affinity matrix for each view. In our case, it costs $\mathcal{O}(npd_v)$ for the $v$th view, while the original method needs $\mathcal{O}(n^2d_v)$. Therefore, in this step, the total cost of our algorithm is $\mathcal{O}(npd)$ and the original requires $\mathcal{O}(n^2d)$.

In the main loop, there are two major steps: construct the augmented view and update the sparse representation matrix for each view. The construction of augmented view includes two sub-steps: performing singular value decomposition of each original view and computing the representation of the augmented view. To complete the first sub-step, the original algorithm needs $\mathcal{O}(n^3V)$ operations comparing to $\mathcal{O}(p^2nV)$ operations for our method. Both approaches finish computing the representation of augmented view with $\mathcal{O}(nk^2V^2)$ operations. Because $n\gg k,V$ in most cases, the cost of constructing the augmented view is $\mathcal{O}(n^3V)$ for the original method comparing to $\mathcal{O}(p^2nV)$ required by ours. Finally, since the number of operations to update a view equals the number of entries in the affinity matrix. Thus, the computational complexity of the last step of the original guided co-training and that of ours are $\mathcal{O}(n^2V)$ and $\mathcal{O}(pnV)$, respectively. We summarize the step-wise computational costs in Table~\ref{summary_conplexity}.

\setlength{\tabcolsep}{4pt}
\begin{table}
\begin{center}
\caption{
Time complexity of each step.
}
\label{summary_conplexity}
%\begin{tabular}{p{4.5cm} p{1.8cm} p{1.5cm}}
\begin{tabular}{lll}
\Xhline{2\arrayrulewidth}
%\noalign{\smallskip}
Step $\qquad\qquad$& Original & Ours\\
%\noalign{\smallskip}
\Xhline{2\arrayrulewidth}
\noalign{\smallskip}
Sampling landmark points & - & $\mathcal{O}(t_1npd)$ \\
Affinity matrix construction & $\mathcal{O}(n^2d)$ & $\mathcal{O}(npd)$ \\
Augmented view construction & $\mathcal{O}(n^3V)$ & $\mathcal{O}(p^2nV)$ \\
Updating views & $\mathcal{O}(n^2V)$ & $\mathcal{O}(pnV)$ \\
\hline
\end{tabular}
\end{center}
\end{table}
\setlength{\tabcolsep}{1.4pt}

%%%%%%%%%%%%%%%%%%%%%%%%%%%%%%%%%%%%%%%%%%%%%%%%%%%%%%%%%%%%%%%%%%%%%%%%%%%%%%%%%%%%%%%%%%%%%%%%%%
%
\section{Experimental Results}
\label{sec:experiment}
%
%%%%%%%%%%%%%%%%%%%%%%%%%%%%%%%%%%%%%%%%%%%%%%%%%%%%%%%%%%%%%%%%%%%%%%%%%%%%%%%%%%%%%%%%%%%%%%%%%%

We carry out experiments on five benchmark datasets to evaluate the usefulness of the proposed algorithm and  demonstrate the effectiveness of our approach by comparing with other spectral clustering algorithms.

\setlength{\tabcolsep}{4pt}
\begin{table*}
\begin{center}
\caption{
Summary of the multi-view datasets.
}
\label{datasets}
\begin{tabular}{p{2.2cm} p{1.8cm} p{2.5cm} p{2.5cm} p{2.5cm} p{2.5cm}}
%\begin{tabular}{cccccc}
\Xhline{2\arrayrulewidth}
%\noalign{\smallskip}
No. & UCI & USPS & MNIST & NUS & Reuters \\
%\noalign{\smallskip}
\Xhline{2\arrayrulewidth}
\noalign{\smallskip}
1 & FAC (216) & Pixel (256) & Pixel (784) & CH (65) & English (21531) \\
2 & FOU (76) & GIST (512) & GIST (512) & CM (226) & France (24893) \\
3 & KAR (64) & HOG (81) & HOG (81) & CORR (145) & German (34279) \\
4 & MOR (6) & - & - & EDH (74) & Italian (15506) \\
5 & PIX (240) & - & - & WT (129) & Spanish (11547) \\
6 & ZER (47) & - & - & - & - \\
\hline
\# of data & 2000 & 11000 & 70000 & 30000 & 18758\\
\# of classes & 10 & 10 & 10 & 31 & 6 \\
\hline
\end{tabular}
\end{center}
\end{table*}
\setlength{\tabcolsep}{1.4pt}

\subsection{Evaluation Metrics}
\label{ssec:evaluation}
We evaluate the clustering results by two standard measures: clustering accuracy (ACC) and normalized mutual information (NMI). For both measures, a higher value indicates a better clustering result. Because of the randomness of $k$-means clustering, the reported value is obtained by averaging the results of 10 tests.

\subsection{Datasets}
\label{ssec:datasets}
Five popular datasets are considered in our experiments. They are described below and summarized in Table~\ref{datasets}.

{\bf UCI Digit (UCI)\footnote{https://archive.ics.uci.edu/ml/datasets/Multiple+Features}:} The dataset is established by extracting 2,000 handwritten numerals from a collection of Dutch utility maps \cite{Lichman:2013}. It consists of 10 categories and each category contains 200 samples. Each sample is characterized by 6 different features: (1) Fourier coefficients of the character shapes; (2) profile correlations; (3) Karhunen-Love coefficients; (4) pixel averages in $2\times3$ windows; (5) Zernike moments; and (6) morphological features. We use all six features for our experiment.

{\bf USPS\footnote{http://www.cs.nyu.edu/~roweis/data.html}:} The dataset is composed of 11,000 8-bit grayscale $16\times16$ images of handwritten digits collected from envelopes by the U.S. Postal Service. We use three different features: (1) The original 256-D grayscale vector; (2) 512-D GIST; (3) 81-D HOG. We adopt the Python package scikit-image to calculate HOG and the package provided by Oliva and Torralba\footnote{http://people.csail.mit.edu/torralba/code/spatialevnelope} to produce the GIST feature.

{\bf MNIST:} This is one of the most popular datasets in computer vision, which has a training set of 60,000 examples, and a test set of 10,000 examples. Each example is a $28\times 28$ image of handwritten digit. Like in dealing with USPS, we consider the 784-D grayscale vector, HOG and GIST. All 70,000 data instances are used in our experiment.

{\bf NUS-WIDE-Object (NUS):} The collection is a lite version of the NUS-WIDE dataset, which is created by Lab for Media Search in National University of Singapore \cite{chua2009nus}. It consists of 30,000 web images belonging to 31 categories. The website\footnote{http://lms.comp.nus.edu.sg/research/NUS-WIDE.htm} offers 5 pre-computed different features: (1) color Histogram; (2) color moment; (3) color correlation; (4) edge distribution; and (5) wavelet texture. All the available features are used in our experiment.

{\bf Reuters Multilingual Text (Reuters):} \cite{amini2009learning} construct this dataset by sampling parts of the Reuters RCV1 and RCV2 \cite{lewis2004rcv1} collections for newswire articles written in 5 languages (English, French, German, Italian and Spanish) and their translations. The documents are made available as feature characteristics in a "bag of words" format from the website\footnote{http://tinyurl.com/z8yr6wf}. Analogous to \cite{kumar2011cor}, we use
Latent Semantic Analysis (LSA) \cite{hofmann1999probabilistic} for dimensionality reduction and reduce the feature dimension to 1500-D.

\subsection{Baseline Algorithms}
\label{ssec:baseline}
We briefly describe the baseline algorithms included in the performance evaluations of our experiment.

{\bf Single Feature Type Landmark-based Spectral Clustering (LSC(\#)):} The landmark-based spectral clustering is an accelerated version of spectral clustering which handles only one single view \cite{chen2011large}. We treat this algorithm as a baseline rather than regular spectral clustering in that the single-view results by LSC can help evaluate how much our method improves the clustering performance with multi-view information. In this experiment, we sample 800 landmark points and construct the representation by setting the number of nearest neighbors as 8 for all five datasets.

{\bf Feature Concatenation Landmark-based Spectral Clustering (ConcatLSC):} We concatenate the features from all views and run the landmark-based spectral clustering (\ie, LSC) on each dataset. The settings, namely, number of landmark points and number of nearest neighbors, are the same as those described above for LSC.

{\bf Multi-view Spectral Clustering via Bipartite Graph (MVSC-B):} To the best of our knowledge, this is the first multi-view spectral clustering algorithm, which is designed for large-scale data \cite{li2015large}. MVSC-B builds a bipartite graph between raw data points and the so-called salient points for Note that the hyper-parameter $r$ in the algorithm is set to $2$ in our experiment.

\subsection{Clustering Performance}
\label{ssec:performance}
The clustering performance of baseline algorithms and our method are lreported in Table \ref{NMI} and Table \ref{ACC}. For fairness, the outcomes of our method and MVSC-B are using Gaussian kernel to measure the similarity with the width parameter being the median among all pair-wise Euclidean distances. Both method are carried out with 600 landmark points, while the number of nearest neighbors is set to 8. It can be readily observed that the best performance, listed in bold digit, are obtained by our method.

We note that the results by MVSC-B are slightly different from those reported in \cite{li2015large}. The differences could be due to two possible factors: (1) The pre-processing of data. We use LSA to pre-process the Reuters dataset. In our experiment, simply using the raw data of Reuters will yield poor results for all methods and the pre-processing is necessary. \cite{kumar2011cor} have observed a similar phenomenon, and also consider performing dimensionality reduction an essential pre-processing for dealing with this dataset. However, we do not know the details of experimental setting conducted in \cite{li2015large}. (2) Different hyper-parameter values. There are some hyper-parameters that are not specified in \cite{li2015large}, including the bandwidth of Gaussian kernel and the power constant $r$ of the weights. Still, our algorithm yields better results under the circumstance that both methods use the same number of landmark points and the same number of nearest neighbors.

We also investigate the impact of using different numbers of landmark points to our method. The outcomes are displayed in Figure~\ref{fig:ACC_NMI}. Generally speaking, the more the landmark points are generated, the better the clustering performance is. However, it is insightful to point out that while the similar phenomenon happens in other techniques \cite{wang2009approximate,chen2011large}, the variance of the outcomes is indeed small. Thus, we indeed do not need a very large set of landmark points to get decent clustering results.

\setlength{\tabcolsep}{4pt}
\begin{table}
\begin{center}
\caption{
NMI by different clustering schemes.
}
\label{NMI}
%\begin{tabular}{p{2.2cm} p{1.8cm} p{2.5cm} p{2.5cm} p{2.5cm} p{2.5cm}}
\begin{tabular}{lccccc}
\Xhline{2\arrayrulewidth}
%\noalign{\smallskip}
Method & UCI & USPS & MNIST & NUS & Reuters\\
%\noalign{\smallskip}
\Xhline{2\arrayrulewidth}
\noalign{\smallskip}
SC(1) & 0.747 & 0.583 & 0.612 & 0.073 & 0.243\\
    SC(2) & 0.759 & 0.740 & 0.583 & 0.091 & 0.255\\
    SC(3) & 0.867 & 0.597 & 0.641 & 0.099 & 0.263\\
    SC(4) & 0.506 & - & - & 0.112 & 0.271\\
    SC(5) & 0.879 & - & - & 0.096 & 0.332\\
    SC(6) & 0.691 & - & - & - & - \\
    \hline
    ConcatLSC & 0.716 & 0.601 & 0.669 & 0.127 & 0.301 \\
    MVSC-B & 0.893 & 0.765 & 0.708 & 0.149 & 0.335\\
    Ours & \bf{0.928} & \bf{0.772} & \bf{0.753} & \bf{0.157} & \bf{0.337}\\
    \hline
    \end{tabular}
\end{center}
\end{table}
\setlength{\tabcolsep}{1.4pt}

\setlength{\tabcolsep}{4pt}
\begin{table}
\begin{center}
\caption{
Accuracy by different clustering schemes.
}
\label{ACC}
%\begin{tabular}{p{2.2cm} p{1.8cm} p{2.5cm} p{2.5cm} p{2.5cm} p{2.5cm}}
\begin{tabular}{lccccc}
\Xhline{2\arrayrulewidth}
%\noalign{\smallskip}
Method & UCI & USPS & MNIST & NUS & Reuters\\
%\noalign{\smallskip}
\Xhline{2\arrayrulewidth}
\noalign{\smallskip}
    SC(1) & 0.744 & 0.542 & 0.591 & 0.140 & 0.341\\
    SC(2) & 0.776 & 0.696 & 0.549 & 0.131 & 0.425\\
    SC(3) & 0.930 & 0.593 & 0.608 & 0.142 & 0.445\\
    SC(4) & 0.376 & - & - & 0.168 & 0.364\\
    SC(5) & 0.940 & - & - & 0.140 & 0.449\\
    SC(6) & 0.725 & - & - & - & -\\
    \hline
    ConcatLSC & 0.758 & 0.555 & 0.697 & 0.184 & 0.472 \\
    MVSC-B & 0.945 & 0.726 & 0.740 & 0.173 & 0.502\\
    Ours & \bf{0.967} & \bf{0.738} & \bf{0.754} & \bf{0.189} & \bf{0.508}\\
    \hline
    \end{tabular}
\end{center}
\end{table}
\setlength{\tabcolsep}{1.4pt}

\begin{figure}[th]
\centering
\begin{tabular} {c}
\includegraphics[width=0.86\linewidth]{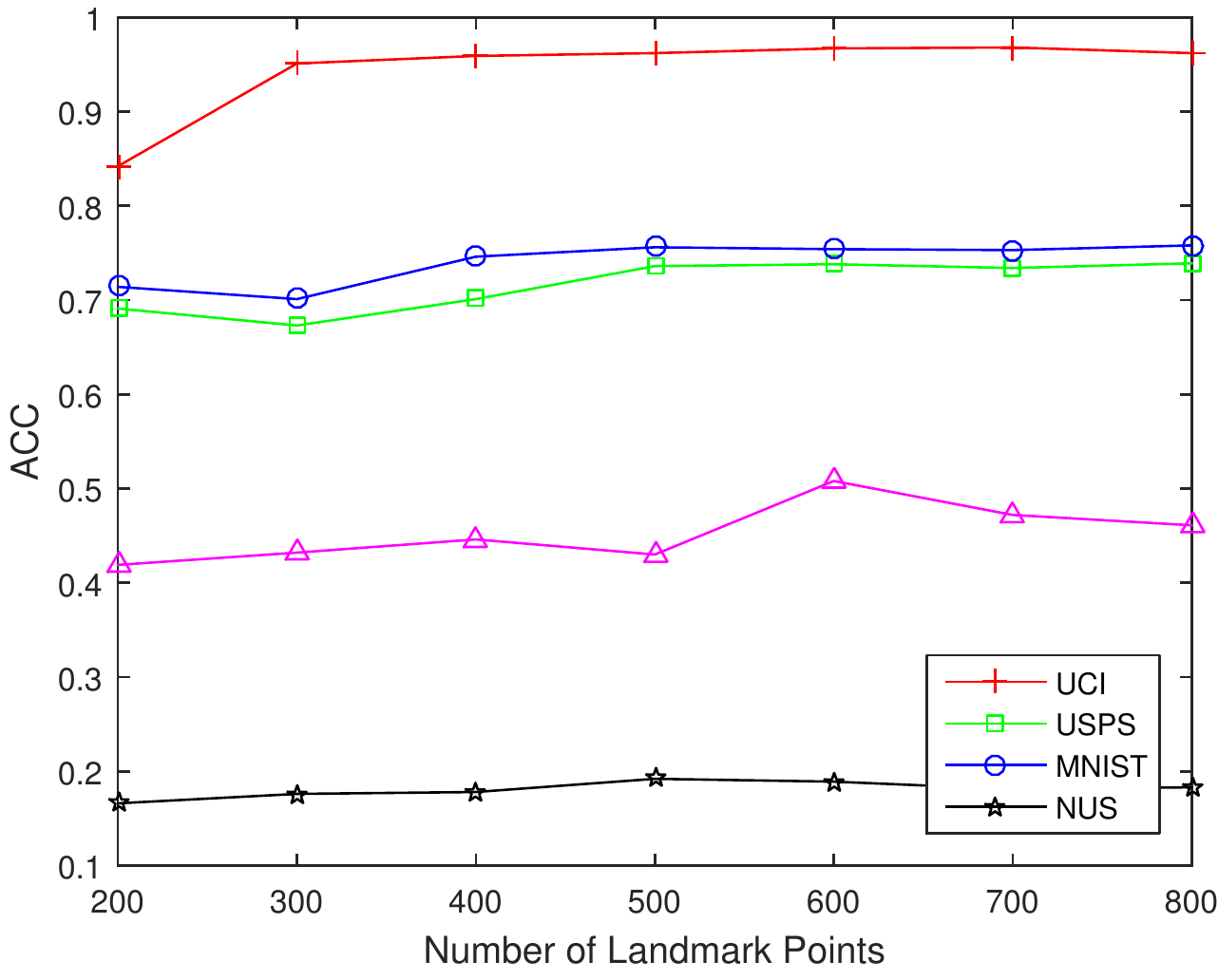} \\
(a) ACC \\
\end{tabular}
\caption{
Performances versus number of landmark points.
}
\label{fig:ACC_NMI}
\end{figure}

%%%%%%%%%%%%%%%%%%%%%%%%%%%%%%%%%%%%%%%%%%%%%%%%%%%%%%%%%%%%%%%%%%%%%%%%%%%%%%%%%%%%%%%%%%%%%%%%%%
%
\section{Conclusion}
\label{sec:conclusion}
%
%%%%%%%%%%%%%%%%%%%%%%%%%%%%%%%%%%%%%%%%%%%%%%%%%%%%%%%%%%%%%%%%%%%%%%%%%%%%%%%%%%%%%%%%%%%%%%%%%%
We have introduced a large-scale extension to the guided co-training approach to tackling multi-view spectral clustering. To divide $n$ given multi-view data points into $k$ clusters, we sample $p$ $(\ll n)$ landmark points (common to all views) to approximate the eigen-decomposition and then construct the augmented view. The main result of this work is to justify the equivalence between computing the approximate affinity matrix of the optimal augmented view and constructing the reduced representation matrix (of rank $k$) to include the most information from all views. The relatedness between the two tasks supports that the augmented view is indeed informative, and has the advantage to   guide the improvements of other original views. As a result, we are able to achieve state-of-the-art clustering performances in five popular benchmark datasets. For future work, we would focus on developing deep-net models to mimic the function of the augmented view and establishing a new architecture to implement the large-scale guided co-training strategy.

\bibliographystyle{IEEEtran}
\bibliography{guided_arXiv}
\end{document}